\documentclass[11pt]{article}

\usepackage[final]{acl}

\usepackage{times}
\usepackage{latexsym}

\usepackage[T1]{fontenc}

\usepackage[utf8]{inputenc}

\usepackage{microtype}

\usepackage{inconsolata}

\usepackage{times}
\usepackage[T1]{fontenc}
\usepackage[utf8]{inputenc}
\usepackage{microtype}
\usepackage{graphicx}
\usepackage{url}
\usepackage{cite}
\usepackage{amsmath,amssymb,amsthm}
\usepackage{booktabs}
\usepackage{subcaption}
\usepackage{bm}
\usepackage{textcomp}
\usepackage{stfloats}

\usepackage{verbatim}
\usepackage{algorithm}
\usepackage{algorithmic}
\newtheorem{definition}{Definition}
\newtheorem{lemma}{Lemma}
\newtheorem{theorem}{Theorem}
\usepackage{pifont}    
\usepackage{multirow}   
\usepackage{subcaption} 
\usepackage{setspace}

\usepackage[utf8]{inputenc}
\usepackage{tcolorbox}
\usepackage{algorithmic}
\usepackage{setspace}
\usepackage{pifont}   
\usepackage{newfloat}
\usepackage{listings}
\DeclareCaptionStyle{ruled}{labelfont=normalfont,labelsep=colon,strut=off} 
\floatstyle{ruled}
\newfloat{listing}{tb}{lst}{}
\floatname{listing}{Listing}
\theoremstyle{plain}

\title{Estimating the Black-box LLM Uncertainty with Distribution-Aligned Adversarial Distillation}

\author{
  \textbf{Huizi Cui\textsuperscript{1}},
  \textbf{Huan Ma\textsuperscript{1}},
  \textbf{Qilin Wang\textsuperscript{2}},
  \textbf{Yuhang Gao\textsuperscript{3}},
  \textbf{Changqing Zhang\textsuperscript{4}\thanks{~Corresponding author.}}
  \\
  \\
  \textsuperscript{1} School of Computer Science and Technology, Tianjin University, China \\
  \textsuperscript{2} School of Future Technology, Tianjin University, China \\
  \textsuperscript{3} Georgia Tech Shenzhen Institute, Tianjin University, China \\
  \textsuperscript{4} School of Artificial Intelligence, Tianjin University, China \\
 \texttt{\{huizicui, mahuan520, wangqilin, yuhang\_gao, zhangchangqing\}@tju.edu.cn} %
}

\begin{document}
\maketitle
\begin{abstract}
Large language models (LLMs) have progressed rapidly in complex reasoning and question answering, yet LLM hallucination remains a central bottleneck that hinders practical deployment, especially for commercial black-box LLMs accessible only via APIs. Existing uncertainty quantification methods typically depend on computationally expensive multiple sampling or internal parameters, which prevents real-time estimation and fails to capture information implicit in the black-box reasoning process. To address this issue, we propose Distribution-Aligned Adversarial Distillation (DisAAD), which introduces a generation-discrimination architecture to guide a lightweight proxy model to learn the high-quality regions of the output distribution of the black-box LLM, thus effectively endowing it with the ability to ``know whether the black-box LLM knows or not''. Subsequently, we use the proxy model to reproduce the specific responses of the black-box LLM and estimate the corresponding uncertainty based on evidence learning. Extensive experiments have verified the effectiveness and promise of our proposed method, indicating that a proxy model even one that only accounts for $1\%$ of the target LLM’s size can achieve reliable uncertainty quantification. Our model and related resources are released at \url{https://github.com/huizi-Cui/DisAAD}.
\end{abstract}

\section{Introduction}

Large language models (LLMs) have made significant progress in recent years, demonstrating outstanding performance in complex reasoning and text generation tasks \citep{kadavath2022language,rawte2023survey,zhang-etal-2025-law}. Despite these remarkable achievements, LLMs are prone to generate seemingly reasonable responses with non-factual or unfaithful information, a phenomenon widely known as LLM hallucinations \citep{shah2024accuracy,banerjee2024llms,tonmoy2024comprehensive}. Moreover, recent works have further indicated that larger and more instructive LLMs usually tend to deceive by pretending to understand, creating a false sense of confidence that makes users easily believe their responses \citep{abbasi2024believe,zhou2024larger,huang2025survey}. Therefore, hallucination issues present a significant barrier to the widespread application of LLMs, particularly in safety-critical fields \citep{chen2025enhancing,perkovic2024hallucinations}.

Uncertainty quantification has emerged as a promising way to mitigate the limitations arising from hallucinations by enabling LLMs to express doubt when generating potentially unreliable responses \citep{huang2024survey,zhang2023enhancing}. High uncertainty indicates that users need to be cautious, since the LLM may be influenced by hallucinations and offer unreliable responses. Depending on computational cost, existing uncertainty quantification methods can be broadly categorized into self-evaluation methods, multi-sample methods and single-sample methods \citep{xiong2024efficient}.

Self-evaluation methods allow an LLM to assess the confidence of its own generated response through internal mechanisms or with the help of additional advanced models \citep{kadavath2022language, kapoor2024large}. However, these methods often fail to produce credible estimation results, and some specific fine-tuning interventions are necessary. Multi-sample methods perceive the diversity within the possible answer space from multiple reference calls, leveraging statistical patterns across generations to identify areas where the model exhibits hesitation or inconsistency \citep{lakshminarayanan2017simple, farquhar2024detecting}. For example, Semantic Entropy quantifies uncertainty in LLMs by measuring consistency across multiple responses generated for the same prompt, and further identifies those inconsistent outputs as potentially unreliable information sources \citep{farquhar2024detecting}. Although multi-sample methods are theoretically well-founded, they face several significant issues: (1) fail to estimate the uncertainty of single response; (2) inefficient in practical applications due to the need for multiple sampling iterations; (3) miss inherent uncertainty when models consistently generate incorrect answers due to knowledge gaps.

In view of the above issues, single-sample methods are developed to estimate the uncertainty of individual sentences by accessing the internal information derived from the LLM (e.g., the next token probability distribution) to estimate the real-time uncertainty \citep{fadeeva2024fact}. LogTokU is a representative method for quantifying token-level uncertainty by treating logits as parameters of the Dirichlet distribution \citep{ma2025estimating}. It provides mathematical evidence that logits offer more accurate uncertainty representations than maximum probability or entropy measurements. However, these methods are not applicable to closed-source LLMs such as GPT-4 and Claude-3, which still dominate in current practical applications \citep{sriramanan2024llm}. Since these LLMs do not provide complete access to their internal mechanisms and parameter states, a fundamental issue arises: \textbf{How well can we predict the real-time uncertainty of black-box LLM only based on the single response?}

Recent research demonstrates that small LLMs often refuse to provide answers to difficult issues, reflecting a better awareness of their knowledge limitations. In contrast, larger and more instructive LLMs (such as GPT-4) tend to give seemingly reasonable but actually incorrect responses more frequently, making them easily overlooked by users \citep{zhou2024larger,steyvers2025large}. Since simple LLMs are more reliable, a natural idea emerges: estimating the uncertainty of black-box LLM by leveraging a smaller LLM. To achieve this, we propose a novel Distribution-Aligned Adversarial Distillation (DisAAD), which introduces a small proxy model to learn how to ``know whether the black-box LLM knows or not'' and enables it to guide uncertainty quantification for the target LLM in downstream tasks. Specifically, we first systematically collect the outputs from the target black-box LLM across diverse prompts and create a comprehensive distillation dataset \citep{zeng2024dald}. Then, the proxy model is specifically optimized within a generation-discrimination architecture to approximate the high-probability regions of the target output distribution. Benefiting from adversarial distillation, we further utilize the distilled proxy model to reproduce the responses of the black-box LLM and estimate real-time uncertainty via evidential deep learning \citep{sensoy2018evidential}. Extensive experiments verify the effectiveness of the proposed method in various question-answer tasks, a distilled proxy model even with only $1\%$ of the target model size can achieve superior response reliability estimation performance. 

The main contributions of our work are summarized as follows: (1) We propose a new paradigm for estimating the uncertainty of black-box LLMs, which not only eliminates the need for accessing model states but also obviates the requirement for multiple response sampling. (2) We propose a new method that enables proxy models to approximate the high-probability regions of the target output distribution, thereby characterizing the uncertainty of black-box LLMs. (3) Through extensive experiments and theoretical analysis, we validate the effectiveness of our proposed method in the detection of hallucinations of LLM, outperforming the strongest baselines in black-box setting with an average improvement of $18.2\%$ in AUROC and $22.9\%$ in AUPR.

\section{Related work}

\subsection{Multi-sample Methods} 

Multi-sample methods evaluate uncertainty by measuring semantic consistency across multiple responses to the same prompt. For instance, Semantic Entropy computes the entropy over a distribution of semantic clusters formed by grouping the sampled responses based on semantic equivalence \citep{farquhar2024detecting}. EigV quantifies uncertainty using a graph-based calculation to estimate how many distinct groups of similar answers exist, which allows it to effectively identify different semantic clusters \citep{lin2023generating}. Furthermore, recent works advance this by integrating the model's internal confidence. CoCoA calculates LLM uncertainty by multiplying the model's confidence in a specific response with the average semantic inconsistency compared to other samples \citep{vashurin2025uncertainty}. Similarly, SAR creates a hybrid measure by combining sentence-level semantic relevance with token-level probability adjustments for a more fine-grained balance of uncertainty \citep{duan2024shifting}.

\subsection{Single-sample Methods} 

Single-sample methods usually achieve LLM uncertainty quantification by leveraging token probabilities, logits or hidden layer activations without requiring additional sampling. The related works include perplexity, negative sequence probability, and mean token entropy \citep{fomicheva2020unsupervised}, along with more advanced techniques that account for the semantic importance. For example, CCP isolates factual uncertainty by analyzing the semantic relationships within the candidate token distribution at each step \citep{fadeeva2024fact}. Focus uses a proxy model to re-weight token-level uncertainty based on semantic properties like keyword importance and entity type \citep{zhang2023enhancing}. Recent research suggests that logits provide more direct insight into model confidence than normalized probabilities, leading to approaches like LogTokU \citep{ma2025estimating}, which treats logits as Dirichlet distribution parameters \citep{abdar2021review}.

\section{Method}

\subsection{Notations}

Given a white-box LLM $\mathcal{M}_\text{W}$ with a vocabulary $\mathcal{V}$, we formalize the next-token prediction process. An input prompt is tokenized as sequence $\boldsymbol{x} = (x_1, \dots, x_L)$. The model autoregressively generates a response $\boldsymbol{y} = (y_1, \dots, y_T)$. At each step $t$, $\mathcal{M}_\text{W}$ processes the context, which comprises the prompt $\boldsymbol{x}$ and previously generated tokens $\boldsymbol{y}_{<t} = (y_1, \dots, y_{t-1})$ to produce logit vector $\boldsymbol{z}_t \in \mathbb{R}^{|\mathcal{V}|}$. The probability of generating token $v_k \in \mathcal{V}$ as the next token $y_t$ is:

\begin{equation} \label{eq21}
P(y_t = v_k | \boldsymbol{x}, \boldsymbol{y}_{<t}; \mathcal{M}_\text{W}) = \frac{\exp(\boldsymbol{z}_{t,k})}{\sum_{j=1}^{|\mathcal{V}|} \exp(\boldsymbol{z}_{t,j})},
\end{equation}

\noindent{where $\boldsymbol{z}_{t,k}$ represents the $k$-th element of the logit vector $\boldsymbol{z}_t$, corresponding to the token $v_k$. The next token $y_t$ would be sampled from the distribution $y_t \sim P(\cdot | \boldsymbol{x}, \boldsymbol{y}_{<t}; \mathcal{M}_\text{W})$. } 

The uncertainty of the white-box LLM $\mathcal{M}_\text{W}$ can be directly evaluated according to the intermediate outputs including probability distribution, logits, etc. In contrast, we consider black-box LLM settings, which are increasingly prevalent in real-world applications. For a given input prompt $\boldsymbol{x}$, the black-box LLM $\mathcal{M}_\text{B}$ simply returns a final response sequence $\boldsymbol{y}_\text{B}$. The objective of our work is to quantify real-time uncertainty by only relying on the single input-response pair $(\boldsymbol{x}, \boldsymbol{y}_\text{B})$ derived from $\mathcal{M}_\text{B}$.

\subsection{Distribution-Aligned Adversarial Distillation}

As proven by recent work, larger and more instructive LLMs like GPT-4 often exhibit overconfidence, while smaller models tend to be better calibrated. Motivated by this, we propose to estimate the black-box uncertainty by leveraging the specifically optimized proxy model with small size. The proposed work is a two-stage method for black-box uncertainty quantification, where the first stage is distribution-aligned adversarial distillation for obtaining the proxy model, and the second stage is proxy-guided LLM uncertainty quantification based on evidential deep learning. 

\subsubsection{Distillation Data Collection}

To align the distribution of our proxy model $\mathcal{M}_\text{p}$ and the target model $\mathcal{M}_\text{B}$, we first construct a small distillation dataset $\mathcal{D}_\text{distill}$. This process begins by creating a diverse set of prompts $\{\boldsymbol{x}^{(i)}\}_{i=1}^N$, collated from both large-scale conversational dataset (open-domain) and task-specific evaluation dataset (in-domain). For each prompt $\boldsymbol{x}^{(i)}$, we query $\mathcal{M}_\text{B}$ multiple times to generate a candidate pool of responses $D_\text{B}^{(i)}$. It constitutes an empirical sampling of the model's true conditional output distribution $P_\text{B}(\boldsymbol{y}|\boldsymbol{x}^{(i)})$, which is often characterized by a long-tailed nature. To ensure that the distillation dataset represents the most characteristic outputs of the target LLM, we select the top $M$ responses from each $D_\text{B}^{(i)}$ that exhibit the highest mutual semantic consistency. This strategy effectively isolates the high-probability regions of the output distribution, and the selected prompt-response pairs $\{ (\boldsymbol{x}^{(i)}, \boldsymbol{y}_\text{B}^{(i,j)}) \}_{i=1, j=1}^{N,M}$ form our final distillation dataset $\mathcal{D}_\text{distill}$.

\subsubsection{Proxy Model Training}

Based on the collected distillation dataset $\mathcal{D}_\text{distill}$, we establish an adversarial training between the proxy model (generator) and the discriminator. As shown in Fig. \ref{flow1}, the proxy model learns to generate the responses consistent with those of the target LLM, while the discriminator learns to distinguish the responses derived from the proxy model and black-box LLM. Our goal is to optimize the proxy model until its generated responses are statistically indistinguishable from those of the target LLM, reaching the point where the discriminator cannot distinguish them.

\begin{figure*}[t]
\centering
\includegraphics[width=0.95\textwidth]{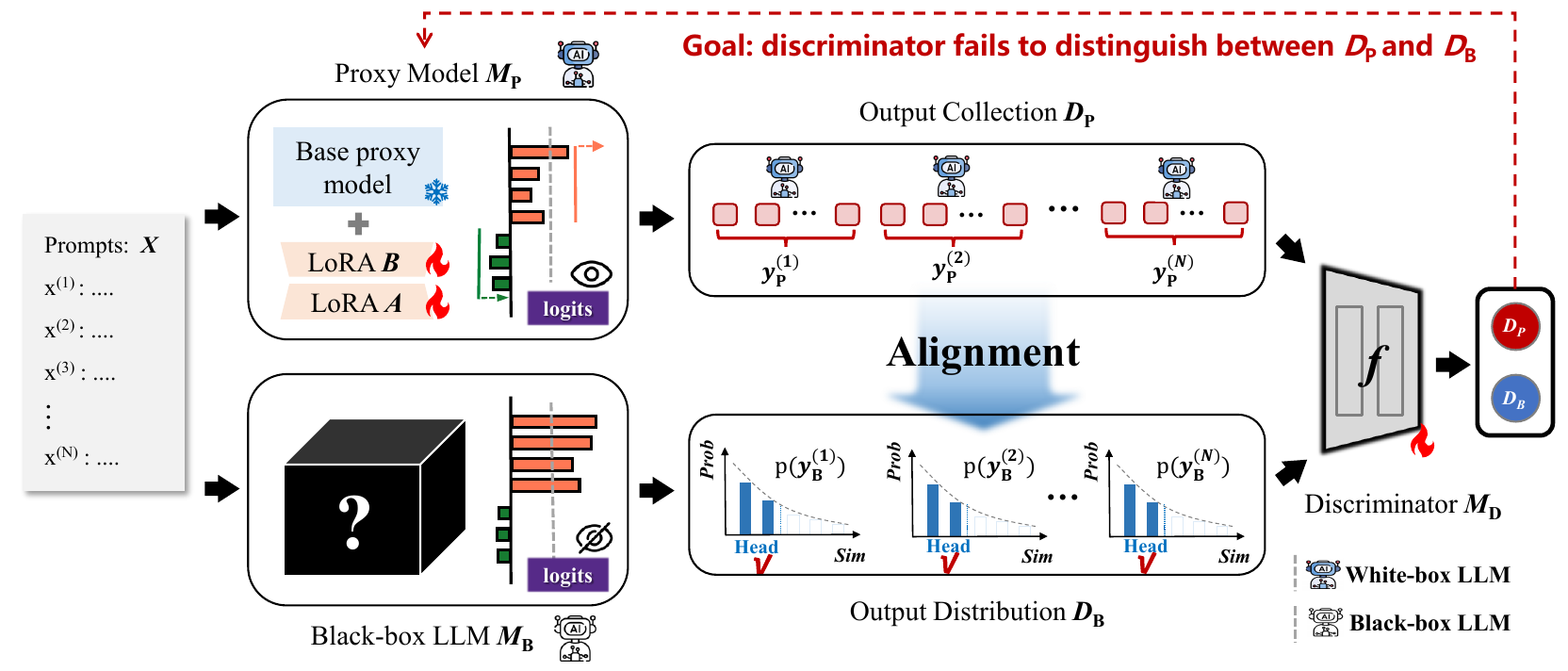} 
\caption{Overview of Distribution-Aligned Adversarial Distillation (DisAAD) framework. For a given small-size prompt set, we first collect the distillation dataset based on the target black-box LLM. Then, the LoRA-based proxy model is trained with the collected distillation dataset under the generation-discrimination architecture. The adversarial training objective is to enable the discriminator to learn to distinguish whether the responses generated by the proxy model are aligned with the high-probability regions from the target output distributions. The adversarial distillation terminates when the discriminator is unable to effectively distinguish the responses of the proxy model and the target black-box LLM. }
\label{flow1}
\end{figure*}

To efficiently fine-tune the proxy model, we employ Low-Rank Adaptation (LoRA) \citep{hu2022lora}. Instead of fine-tuning all of the model's parameters, LoRA freezes the original weights ($W_0$) and injects trainable low-rank matrices ($B$ and $A$) into each layer of the model. The effective weights are represented as:

\begin{equation}
W = W_0 + BA,
\end{equation}

\noindent{where $W_0 \in \mathbb{R}^{d \times d}$ represents the pre-trained weight matrix, $B \in \mathbb{R}^{d \times r}$ and $A \in \mathbb{R}^{r \times d}$ are trainable low-rank matrices with rank $r \ll d$. We denote the base proxy model as $\mathcal{M}_\text{p}$, while LoRA-enhanced proxy model as $\mathcal{M}_\text{p}$, parameterized by its trainable LoRA weights $\theta$.}

In our proposed {DisAAD} framework, we aim to train a proxy model $\mathcal{M}_\text{p}$ to approximate the output behavior of a black-box LLM, denoted as $\mathcal{M}_\text{B}$. To this end, we define the training objective of the proxy model as minimizing the following loss:
\begin{equation}
\min_{\theta} \mathcal{L}(\theta) = \mathcal{L}_\text{task}(\theta) + \lambda \mathcal{L}_\text{reg}(\theta),
\end{equation}

\noindent{where $\theta$ denotes the parameters of the proxy model $\mathcal{M}_\text{p}$, and $\lambda > 0$ is a regularization coefficient. The loss consists of two parts:}

\begin{itemize}
    \item \textbf{Task Loss} $\mathcal{L}_\text{task}$: A standard distillation loss that aligns the output of the proxy model with the black-box model at the token level.
    \item \textbf{Regularization Loss} $\mathcal{L}_\text{reg}$: A sequence-level alignment constraint that enforces the proxy model's output to be indistinguishable from the black-box model by a discriminator.
\end{itemize}

\noindent\textbf{Token-Level Distillation Loss (updating $\mathcal{M}_\text{p}$).}  
We adopt a next-token prediction loss to encourage the proxy model to imitate the output of the black-box LLM. Given a dataset of $N$ prompts $\{\boldsymbol{x}^{(i)}\}_{i=1}^N$, each associated with $M$ sampled responses from the black-box model $\{\boldsymbol{y}_\text{B}^{(i,j)}\}_{j=1}^M$, the task loss is defined as:

\begin{equation}
\small
\mathcal{L}_\text{task}(\theta) = -\frac{1}{NM} \sum_{i=1}^{N} \sum_{j=1}^{M} \sum_{t = l(\boldsymbol{x}^{(i)}) + 1}^{l(\boldsymbol{x}^{(i)}) + l(\boldsymbol{y}_\text{B}^{(i,j)})} \log P_\theta(y_t \mid y_{<t}),
\end{equation}

\noindent{where $l(\boldsymbol{x}^{(i)})$ denotes the length of the prompt and $l(\boldsymbol{y}_\text{B}^{(i,j)})$ is the length of the target response. Following common instruction tuning practices \citep{zeng2024dald}, we mask gradients from prompt tokens to prevent interference during learning.}

\noindent\textbf{Discriminator-Based Regularization (updating $\mathcal{M}_\text{p}$).}  
To further align the proxy model with the target model at the sequence level, we introduce a discriminator $\mathcal{M}_\text{D}$ parameterized by $\phi$. The discriminator receives a prompt-response pair and attempts to distinguish whether the response is from the black-box model or the proxy model. The regularization loss encourages the proxy model to generate outputs that the discriminator cannot confidently classify as fake:

\begin{equation}
\mathcal{L}_\text{reg}(\theta) = -\frac{1}{NM} \sum_{i=1}^{N} \sum_{j=1}^{M} \log \mathcal{M}_\text{D}(\boldsymbol{x}^{(i)}, \boldsymbol{y}_\text{P}^{(i,j)}; \phi),
\end{equation}

\noindent{where $\boldsymbol{y}_\text{P}^{(i,j)}$ denotes the response generated by the current proxy model $\mathcal{M}_\text{p}$ for prompt $\boldsymbol{x}^{(i)}$, and $\mathcal{M}_\text{D}(\cdot)$ is the discriminator's estimated probability that the generated response is from the target model.}

\noindent\textbf{Discriminator Loss (updating $\mathcal{M}_\text{D}$).}  
The discriminator is trained to classify whether a response is generated by the black-box model or by the proxy. Its objective is to maximize classification accuracy over real (target) and fake (proxy) samples:

\begin{multline}
\mathcal{L}_\text{D}(\phi) = -\frac{1}{NM} \sum_{i=1}^{N} \sum_{j=1}^{M} \log \mathcal{M}_\text{D}(\boldsymbol{x}^{(i)}, \boldsymbol{y}_\text{B}^{(i,j)}; \phi) \\
-\frac{1}{NM} \sum_{i=1}^{N} \sum_{j=1}^{M} \log \left(1 - \mathcal{M}_\text{D}(\boldsymbol{x}^{(i)}, \boldsymbol{y}_\text{P}^{(i,j)}; \phi)\right).
\end{multline}

\noindent\textbf{Training Procedure.}  
The proxy model and discriminator are trained in an alternating fashion:
\begin{enumerate}
    \item Fix the discriminator, update the proxy model $\theta$ by minimizing $\mathcal{L}(\theta)$.
    \item Fix the proxy model, update the discriminator $\phi$ by minimizing $\mathcal{L}_\text{D}(\phi)$.
\end{enumerate}

This iterative training allows the proxy model to gradually learn to produce responses that are both locally (token-wise) and globally (sequence-wise) aligned with the black-box LLM. The discriminator acts as an adaptive regularizer, improving the semantic fidelity of the proxy model's generations. For more details please refer to the Appendix C.2.

\subsection{Proxy-guided Uncertainty Quantification}

As shown in Fig. \ref{flow2}, we first utilize the proxy model to reproduce the responses of black-box LLMs, then extract the corresponding logits and estimate the real-time uncertainty based on evidential deep learning \citep{sensoy2018evidential,han2022trusted}.

\subsubsection{Logits as Evidence}

\begin{figure*}[h]
\centering
\includegraphics[width=0.95\textwidth]{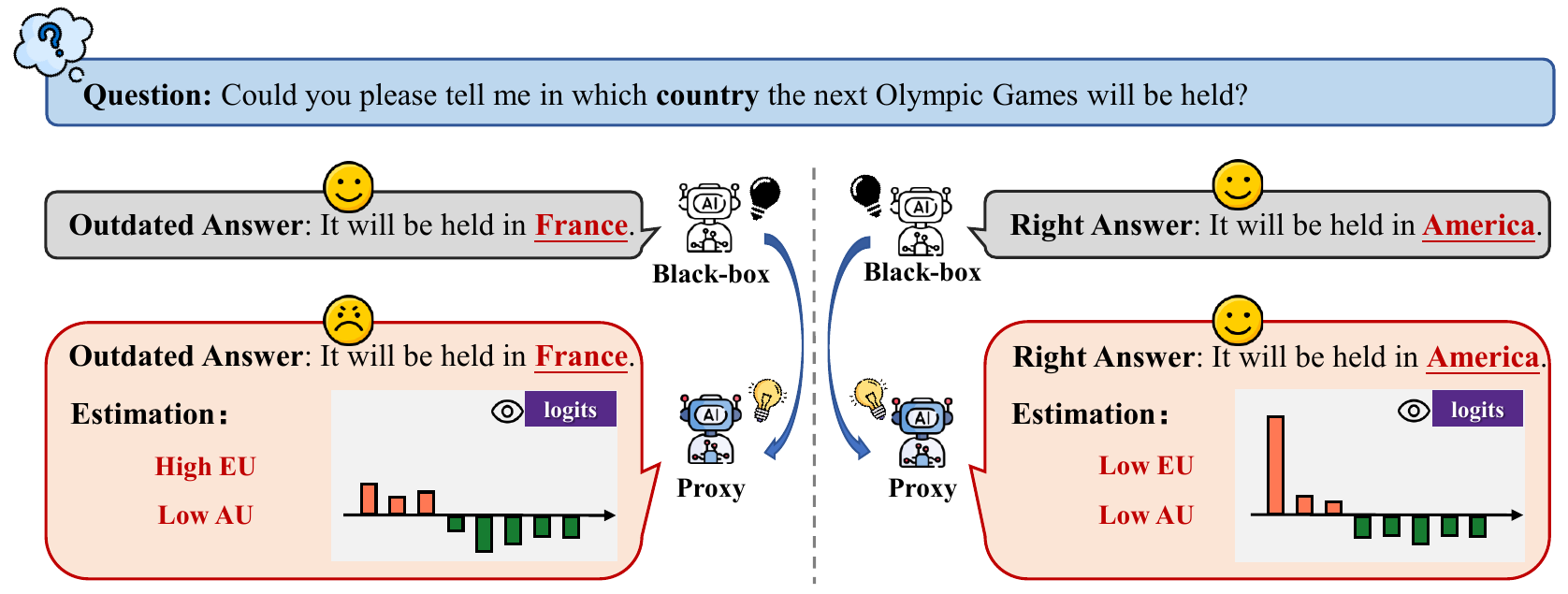} 
\caption{Overview of proxy-guided uncertainty quantification. For a given response from the target LLM, we first utilize the proxy model to reproduce the responses, then extract the corresponding logits and estimate the decoupled uncertainty. For an incorrect or outdated answer (e.g., ``France''), the proxy model would identify this scenario as High EU and Low AU, which indicates that the response of the target model is unreliable. Conversely, for a correct answer (e.g., ``America''), the proxy model would identify this scenario as Low EU and Low AU, which indicates that the response of the target model is reliable.}
\label{flow2}
\end{figure*}

Probability-based uncertainty estimation methods perform poorly because softmax normalization eliminates absolute evidence scale. Logits, however, can more flexibly capture both epistemic uncertainty and aleatoric uncertainty \citep{ma2025estimating}. Given a prompt $\boldsymbol{x}$ and response $\boldsymbol{y}_\text{B}$, we process the pair through our proxy model $\mathcal{M}_\text{p}$ to extract token-level logits $\boldsymbol{z}_t \in \mathbb{R}^{|\mathcal{V}|}$. To reduce noise from extremely low logits, we select only the top $K$ token candidates with the largest logits to model a Dirichlet distribution \citep{tang2024top}:

\begin{equation}
\alpha_k = f(\boldsymbol{z}_{t,k}), \quad \alpha_0 = \sum_{k=1}^{K} \alpha_k
\end{equation}

\noindent{where $\boldsymbol{z}_{t,k} \in \mathbb{R}$ denotes the logit value for the $k$-th token candidate at decoding step $t$, $f(\cdot)$ is the ReLU activation function that converts logits to evidence parameters, and $\alpha_0$ represents the total evidence strength.} 

\subsubsection{Aleatoric Uncertainty (AU)}

Aleatoric uncertainty, also known as data uncertainty, reflects the peak characteristic of output distributions. A lower AU indicates a peaked distribution where the probability mass is mainly concentrated on a single class, while the higher AU reflects the mass distributed uniformly across different classes. The specific definition is listed as follows:

\begin{equation}
\text{AU}(u_t) = -\sum_{k=1}^{K} \frac{\alpha_k}{\alpha_0} (\psi(\alpha_k + 1) - \psi(\alpha_0 + 1)),
\end{equation}

\noindent{where $\psi(\cdot)$ denotes the digamma function defined as $\psi(x) = \frac{d}{dx}\log\Gamma(x)$, $u_t = \{\alpha_1, \alpha_2, ..., \alpha_K\}$ is the set of evidence parameters at position $t$ which represents the Dirichlet distribution parameters derived from logits.}

\subsubsection{Epistemic Uncertainty (EU)}

Epistemic uncertainty, also known as model uncertainty, reflects the model's overall confidence in its prediction regardless of which specific token is selected. It is measured by:

\begin{equation}
\text{EU}(u_t) = \frac{K}{\sum_{k=1}^{K} (\alpha_k + 1)},
\end{equation}

\noindent{where $\sum_{k=1}^{K} (\alpha_k + 1)$ is a smoothed measure of total evidence strength. A lower EU corresponds to high confidence prediction, while higher EU indicates knowledge gaps.}

\section{Experiments}

\subsection{Response Reliability}

Consistent with concurrent work \citep{duan2024shifting, wang2025beyond}, we determine sentence reliability by focusing on the most uncertain tokens. The reliability of a given token $u_t$ is defined as:

\begin{equation}
R(u_t) = - \text{AU}(u_t) \cdot \text{EU}(u_t),
\end{equation}

\noindent{where $R(u_t)$ captures the synergistic impact between aleatoric and epistemic uncertainty.} Then the overall response reliability $R_{\text{response}}$ is calculated by averaging the reliability values of the $K^*$ least reliable tokens:

\begin{equation}
R_{\text{response}} = \frac{1}{K^*} \sum_{u_t \in T_{K^*}} R(u_t),
\end{equation}

\noindent{where $T_{K^*}$ represents the set of $K^*$ tokens with the lowest reliability values.}

\subsection{Evaluation Metrics and Datasets}

The evaluation is formulated as a binary classification task, and conducted on question-answering benchmarks from different domains, including life sciences (BioASQ \citep{tsatsaronis2015overview}), truthfulness (TruthfulQA \citep{lin2021truthfulqa}) and knowledge seeking (TriviaQA \citep{joshi2017triviaqa}). The responses generated by LLM with ``$\text{BLEURT}$$>$0.5'' or ``LLM-Judge=1'' are considered the correct answers \citep{xiong2024efficient}. The performance is quantified by AUROC, AUPR and ECE. For further details, please refer to the Appendix D.1.

\subsubsection{Baseline Methods}

We compare the proposed work with several SOTA black-box LLM uncertainty quantification methods, including Semantic Entropy (SE) and Discrete Semantic Entropy (DSE) \citep{farquhar2024detecting}, LN-Entropy (LNE) \citep{malinin2020uncertainty}, Lexical Similarity (LeS) \citep{lin2023generating} and EigV \citep{zhou2024larger}. In addition, to show that our work can offer competitive performance without requiring access to the LLM's internal states, some white-box methods including probability-based method, entropy-based method and LogTokU \citep{ma2025estimating, xiong2024efficient} are also used for comparative analysis.

\subsection{Distillation Dataset Collection}

We construct the distillation set by mixing half from in-domain evaluation prompts and half from out-of-domain conversational prompts to balance domain relevance and generalization. For each prompt, we draw $10$ high-quality candidates from the target LLM. DisAAD yields an effective and computationally efficient proxy model using only $1K$ distillation samples ($100$ prompts). For additional information, please refer to Appendices C.1 and D.5.

\subsection{Model Training}

For the generator (proxy model), we employ LLaMA series models \citep{zheng2024llamafactory} fine-tuned via Low-Rank Adaptation (LoRA) with rank=$32$, alpha=$64$, and dropout=$0.1$. Following the LLaMA architecture, we target all attention and feed-forward projections (``q\_proj'', ``v\_proj'', ``k\_proj'', ``o\_proj'', ``gate\_proj'', ``down\_proj'', ``up\_proj''). The model is optimized via AdamW with learning rate $1 \times 10^{-4}$ during training. For the discriminator, we utilize a GPT-2 encoder with the final three transformer layers unfrozen, optimized via AdamW at learning rate $1 \times 10^{-5}$. Throughout the adversarial distillation process, the discriminator serves a dual purpose by monitoring training sufficiency and providing quantitative evaluation metrics, effectively determining when the proxy model has successfully captured the characteristics of the target LLM's distribution. For more details, please refer to Appendix D.3.

\subsection{Main Results}

\begin{table*}[h]
\centering
\resizebox{2.0\columnwidth}{!}{
\begin{tabular}{ccccccccc}
\toprule
\multirow{2.5}{*}{Target} & \multirow{2.5}{*}{Method} & \multirow{2.5}{*}{$O(1)$} & \multicolumn{2}{c}{TruthfulQA} & \multicolumn{2}{c}{BioASQ} & \multicolumn{2}{c}{TriviaQA} \\
\cmidrule(lr){4-5} \cmidrule(lr){6-7} \cmidrule(lr){8-9}
& & & AUROC$\uparrow$ & AUPR$\uparrow$ & AUROC$\uparrow$ & AUPR$\uparrow$ & AUROC$\uparrow$ & AUPR$\uparrow$ \\
\midrule
\multicolumn{9}{c}{\textbf{Black-box setting}}\\
\midrule
\multirow{5}{*}{LLaMA2-70B} 
& LEN & \ding{55} & $49.19^{\pm 4.12}$ & $46.52^{\pm 3.50}$ & $54.66^{\pm 4.18}$ & $48.11^{\pm 3.33}$ & $64.24^{\pm 3.90}$ & $84.21^{\pm 4.52}$ \\
& SE & \ding{55} & $53.93^{\pm 3.98}$ & $46.73^{\pm 4.45}$ & $59.94^{\pm 3.01}$ & $48.25^{\pm 4.29}$ & $\bm{65.80}^{\pm 4.82}$ & $82.91^{\pm 5.60}$ \\
& DSE & \ding{55} & $52.52^{\pm 3.05}$ & $45.04^{\pm 4.61}$ & $59.83^{\pm 3.03}$ & $44.43^{\pm 3.58}$ & $65.28^{\pm 4.85}$ & $81.62^{\pm 6.69}$ \\
& LES & \ding{55} & $64.19^{\pm 4.80}$ & $76.39^{\pm 5.95}$ & $52.85^{\pm 4.15}$ & $45.73^{\pm 3.51}$ & $46.22^{\pm 4.40}$ & $\bm{90.26}^{\pm 3.05}$ \\
& EigV & \ding{55} & $66.21^{\pm 3.75}$ & $77.03^{\pm 4.88}$ & $53.15^{\pm 5.10}$ & $46.22^{\pm 4.40}$ & $55.88^{\pm 3.51}$ & $86.95^{\pm 4.12}$ \\
& DisAAD & \ding{51} & $\bm{80.15}^{\pm 1.12}$ & $\bm{78.07}^{\pm 1.25}$ & $\bm{70.46}^{\pm 2.40}$ & $\bm{78.74}^{\pm 1.19}$ & $65.03^{\pm 3.86}$ & $88.52^{\pm 2.20}$ \\
\midrule
\multirow{5}{*}{GPT-4} 
& LEN & \ding{55} & $53.57^{\pm 3.10}$ & $67.83^{\pm 4.01}$ & $60.90^{\pm 3.95}$ & $48.27^{\pm 4.22}$ & $50.89^{\pm 3.20}$ & $89.16^{\pm 4.15}$ \\
& SE & \ding{55} & $49.62^{\pm 4.23}$ & $63.54^{\pm 3.15}$ & $65.84^{\pm 4.70}$ & $57.80^{\pm 5.98}$ & $57.56^{\pm 4.01}$ & $91.54^{\pm 6.99}$ \\
& DSE & \ding{55} & $51.29^{\pm 3.18}$ & $64.66^{\pm 5.09}$ & $60.51^{\pm 3.98}$ & $47.89^{\pm 4.25}$ & $54.89^{\pm 3.11}$ & $90.98^{\pm 4.03}$ \\
& LES & \ding{55} & $65.10^{\pm 4.75}$ & $77.20^{\pm 5.90}$ & $53.15^{\pm 3.10}$ & $46.05^{\pm 3.40}$ & $47.33^{\pm 4.30}$ & $91.10^{\pm 3.00}$ \\
& EigV & \ding{55} & $66.05^{\pm 3.72}$ & $76.81^{\pm 4.90}$ & $63.90^{\pm 3.08}$ & $56.01^{\pm 6.48}$ & $55.91^{\pm 3.45}$ & $90.15^{\pm 4.08}$ \\
& DisAAD & \ding{51} & $\bm{80.78}^{\pm 2.83}$ & $\bm{90.79}^{\pm 1.77}$ & $\bm{69.93}^{\pm 3.42}$ & $\bm{87.42}^{\pm 4.90}$ & $\bm{68.93}^{\pm 1.60}$ & $\bm{95.80}^{\pm 2.50}$ \\
\midrule
\multicolumn{9}{c}{\textbf{White-box setting}}\\
\midrule
\multirow{4}{*}{LLaMA2-70B} 
& Probability & \ding{51} & $65.89^{\pm 4.70}$ & $59.90^{\pm 3.20}$ & $57.07^{\pm 4.05}$ & $67.20^{\pm 3.66}$ & $60.06^{\pm 4.99}$ & $85.89^{\pm 3.40}$ \\
& Entropy & \ding{51} & $67.32^{\pm 3.65}$ & $61.21^{\pm 4.15}$ & $59.67^{\pm 3.99}$ & $69.24^{\pm 4.59}$ & $59.32^{\pm 3.01}$ & $85.56^{\pm 4.42}$ \\
& LogTokU & \ding{51} & $72.17^{\pm 4.40}$ & $68.74^{\pm 3.80}$ & $58.98^{\pm 4.00}$ & $69.43^{\pm 3.58}$ & $64.40^{\pm 4.88}$ & $86.62^{\pm 3.35}$ \\
& DisAAD & \ding{51} & $\bm{80.15}^{\pm 1.12}$ & $\bm{78.07}^{\pm 1.25}$ & $\bm{70.46}^{\pm 2.40}$ & $\bm{78.74}^{\pm 1.19}$ & $\bm{65.03}^{\pm 3.86}$ & $\bm{88.52}^{\pm 2.20}$ \\
\midrule
\multirow{4}{*}{Qwen3-32B} 
& Probability & \ding{51} & $65.70^{\pm 4.72}$ & $73.10^{\pm 3.66}$ & $63.93^{\pm 4.85}$ & $83.33^{\pm 3.20}$ & $63.78^{\pm 3.90}$ & $79.12^{\pm 4.80}$ \\
& Entropy & \ding{51} & $65.90^{\pm 3.70}$ & $74.07^{\pm 4.60}$ & $64.67^{\pm 3.81}$ & $83.86^{\pm 4.15}$ & $66.58^{\pm 4.77}$ & $80.57^{\pm 3.75}$ \\
& LogTokU & \ding{51} & $66.02^{\pm 1.69}$ & $74.97^{\pm 2.55}$ & $64.47^{\pm 4.82}$ & $83.41^{\pm 3.18}$ & $\bm{75.22}^{\pm 3.45}$ & $\bm{86.78}^{\pm 4.33}$ \\
& DisAAD & \ding{51} & $\bm{74.78}^{\pm 2.30}$ & $\bm{83.10}^{\pm 3.10}$ & $\bm{65.71}^{\pm 2.75}$ & $\bm{84.80}^{\pm 1.05}$ & $69.35^{\pm 1.65}$ & $81.25^{\pm 2.70}$ \\
\bottomrule 
\end{tabular}}
\caption{Reliability estimation performance in the QA tasks of different domains. Response correctness for TruthfulQA is determined by ``$\text{BLEURT}$$>$0.5'', while the others are based on ``LLM-Judge=1''. $O(1)$ reflects the complexity of the response sampling process (multi-sample or single-sample). }
\label{tab1}
\end{table*}

\subsubsection{Results in Black-box Settings}

Table \ref{tab1} presents our main comparative results in black-box settings. The findings demonstrate that the proposed DisAAD achieves the best performance in response reliability estimation in most cases. Besides, a key advantage is its ability to perform real-time reliability estimation on a single response, which avoids the significant computational overhead of the multi-sample generation required by the other comparison methods. Specifically, when GPT-4 serves as the target LLM \citep{achiam2023gpt}, DisAAD achieves an average AUROC of $0.7321$ and AUPR of $0.9134$, substantially outperforming the best-performing baselines by $18.2\%$ and $22.9\%$, respectively. In particular, this is achieved using a proxy model with only $1\%$ of the target size. Similar performance is observed with other target LLMs. These results fully establish that DisAAD offers a more effective paradigm for LLM uncertainty quantification. For more details please refer to Appendix E.

\subsubsection{Results in White-box Settings}

To further investigate the performance of our proposed method, some representative white-box methods are used for comparison analysis. The results shown in Table \ref{tab1} indicate that our method not only performs comparably to these methods, but often outperforms them in most cases, despite the fact that the proposed DisAAD does not require internal information derived from the target LLM. Specifically, with LLaMA2-70B as target LLM, DisAAD improves the average AUROC by $6.7\%$ over the strongest baseline LogTokU. Compared to white-box methods that directly inherit the potentially flawed signal from the target's internal states, DisAAD achieves the uncertainty estimation by leveraging a smaller distilled proxy model, which appears to provide a more reliable result. 

\subsection{Further analysis}

\subsubsection{Efficiency of the Small-Size Distillation Dataset and Discriminator}

To demonstrate the training efficiency of DisAAD, we analyze the effectiveness of small-size distillation datasets and our generation-discrimination architecture (LLaMA3-3B and GPT-2). Fig. \ref{ab1}(a) shows both AUROC and AUPR rapidly improve and stabilize after approximately $50$ distillation prompts, with a slight pre-convergence decrease suggesting the proxy model captures essential uncertainty patterns before minor overfitting. The consistent ECE across different dataset sizes in Fig. \ref{ab1}(b) further support this finding. Fig. \ref{ab2}(a) illustrates the prediction gap (discriminator confidence difference between proxy and target outputs) narrowing to approximately $0.0050$, indicating the discriminator can no longer effectively distinguish between the models' outputs. This serves as an auxiliary indicator that the proxy model has been adequately trained. Additionally, Fig. \ref{ab2}(b) visualizes the semantic similarity distribution for validation samples, showing a clear shift toward similarity scores exceeding $0.9$, confirming high consistency between the proxy and target models. For additional theoretical proof and experimental results, refer to Appendices A and E.1.

\begin{figure}[t]
\centering
\includegraphics[width=1\columnwidth]{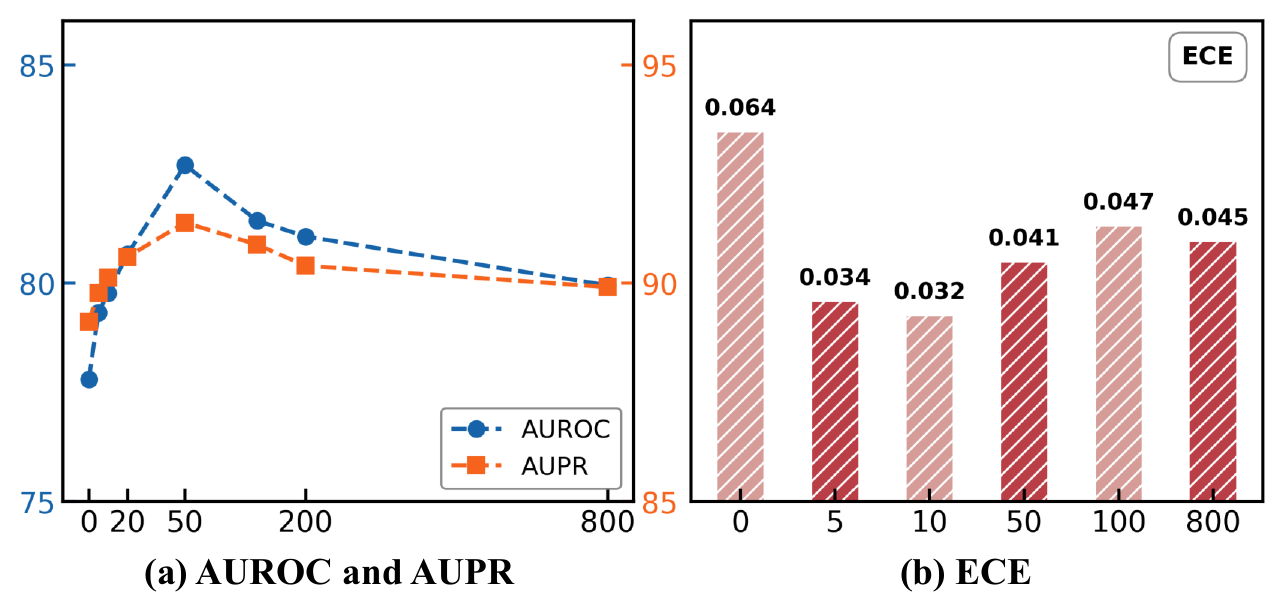} 
\caption{Reliability estimation performance from our distilled proxy model with different numbers of distillation prompts. Left: AUROC and AUPR results. Right: ECE results.}
\label{ab1}
\end{figure}

\begin{figure}[t]
\centering
\includegraphics[width=1\columnwidth]{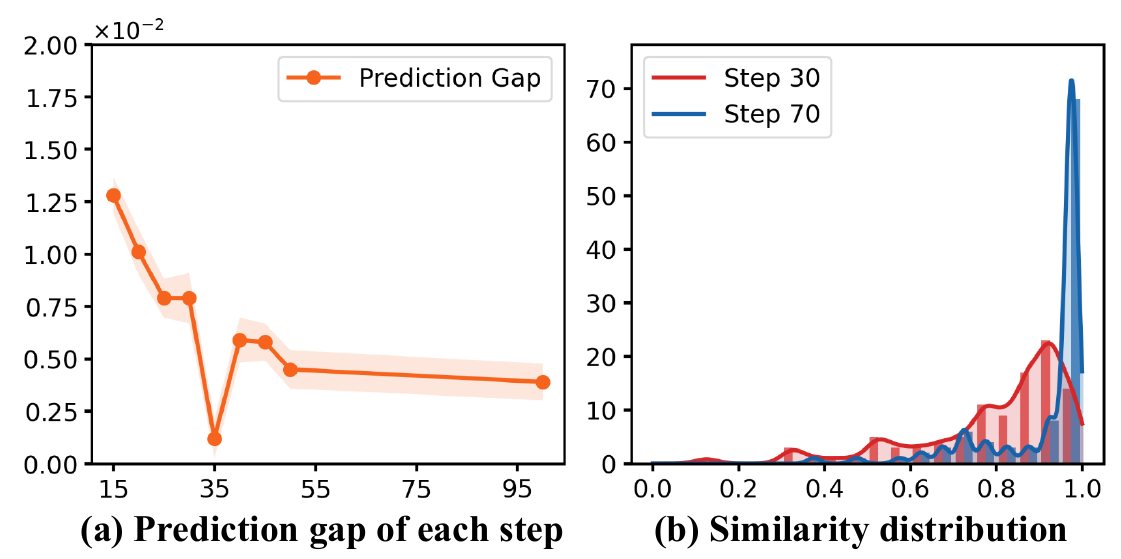} 
\caption{The performance of the discriminator in our DisAAD. Left: The prediction gap between the proxy model and the target LLM of each step. Right: The comparison semantic similarity distribution of Epoch 1.}
\label{ab2}
\end{figure}

\subsubsection{Effectiveness of Distilled Proxy Model}

To verify the effectiveness of the proposed adversarial distillation, we compare the reliability estimation performance of the base proxy model and our distilled proxy model on TruthfulQA. As demonstrated in Table \ref{tab3}, the distilled proxy model consistently shows better performance in AUROC and AUPR than the base proxy model. Furthermore, the ECE results indicate that the improvement of the discriminative power does not always come at the cost of calibration, e.g., DisAAD achieves a better ECE for both the Qwen3-32B \citep{yang2025qwen3technicalreport} and GPT-4. The results verify that the adversarial distillation process can significantly enhance the ability of the proxy model to accurately capture the output characteristics of the target black-box LLM, so that our method can provide more accurate uncertainty quantification results.

\begin{table}[h]
\centering
\resizebox{0.95\columnwidth}{!}{
\begin{tabular}{clccc}
\toprule
\multirow{2.5}{*}{Target} & \multirow{2.5}{*}{Proxy} & \multicolumn{3}{c}{TruthfulQA} \\
\cmidrule(lr){3-5}
& &  AUROC$\uparrow$ & AUPR$\uparrow$ & ECE$\downarrow$ \\
\midrule
\multirow{2}{*}{LLaMA2-70B} 
& Base$^*$ & 0.7801 & 0.7608 & 0.0779 \\
& DisAAD & \textbf{0.8015} & \textbf{0.7807} & \textbf{0.0741} \\
\midrule
\multirow{2}{*}{Qwen3-32B} 
& Base$^*$ & 0.7228 & 0.7793 & 0.0749 \\
& DisAAD & \textbf{0.7478} & \textbf{0.8310} & \textbf{0.0664} \\
\midrule
\multirow{2}{*}{GPT-4} 
& Base$^*$ & 0.7780 & 0.8912 & \textbf{0.0636} \\
& DisAAD & \textbf{0.8078} & \textbf{0.9079} & 0.0685 \\
\bottomrule
\end{tabular}}
\caption{Performance comparison of base and distilled proxy models on the TruthfulQA dataset. ``Base$^*$'' refers to the vanilla LLaMA3-3B, while ``DisAAD'' refers to the same architecture trained with our proposed method.}
\label{tab3}
\end{table}

\subsubsection{Flexibility in the Choice of Proxy Model}

We evaluated LLaMA3 models of varying sizes (1B, 3B, and 8B) as distilled proxy models for reliability estimation. Table \ref{tab4} reveals a non-monotonic relationship between model size and performance, with the 3B model achieving optimal results. This model shows a $3.2\%$ AUROC improvement over the 1B variant, which lacks sufficient capacity for complex uncertainty modeling. Conversely, the 8B model demonstrates a $2.7\%$ AUROC reduction compared to the 3B model, likely because larger LLMs tend to exhibit overconfidence even when uncertain \citep{zhou2024larger,steyvers2025large}. These findings suggest that the ideal proxy model requires moderate scale to effectively capture uncertainty patterns without inheriting the overconfidence issues of larger LLMs.

\begin{table}[h]
\centering
\resizebox{.99\columnwidth}{!}{
\begin{tabular}{ccccc}
\toprule
\multirow{2.5}{*}{Target} & \multirow{2.5}{*}{Proxy} & \multicolumn{3}{c}{TruthfulQA}  \\
\cmidrule(lr){3-5}
& & AUROC$\uparrow$ & AUPR$\uparrow$ & ECE$\downarrow$ \\
\midrule
\multirow{3}{*}{LLaMA2-70B} 
& LLaMA3-1B & 0.7768 & 0.7596 & \textbf{0.0461} \\
& LLaMA3-3B & \textbf{0.8015} & \textbf{0.7807} & 0.0741 \\
& LLaMA3-8B & 0.7797 & 0.7367 & 0.0723 \\
\bottomrule
\end{tabular}}
\caption{Reliability estimation performance of the distilled proxy model with different scales.}
\label{tab4}
\end{table}

\section{Conclusion}

In this work, we introduce Distribution-Aligned Adversarial Distillation (DisAAD), a novel framework for estimating the uncertainty of black-box LLMs using a lightweight proxy model. Unlike existing methods that require multiple queries or access to internal model parameters, our approach enables real-time uncertainty quantification solely based on a single input-response pair. Through adversarial training and distribution alignment, the proxy model effectively learns to approximate the high-probability output regions of the target black-box LLM, thus acquiring the ability to discern response reliability. Extensive experiments on multiple question-answering benchmarks demonstrate that DisAAD achieves state-of-the-art performance in black-box LLM uncertainty estimation, paving the way for a safer and more trustworthy deployment in real-world applications.

\section*{Limitations}

In this work, we mainly focus on token-level uncertainty estimation of black-box LLMs, which may overlook higher-level semantic or contextual inconsistencies in the LLM's full responses that are not captured by isolated token analysis. Besides, the proposed work relies on a prior adversarial distillation step to fine-tune the proxy model, ensuring that the proxy model can align with the target black-box LLM's output distribution before uncertainty assessment. Though the step of adversarial distillation would incur additional costs, both experimental results and theoretical analysis demonstrate that the proposed work can train an effective proxy model using only limited distillation samples.

\section*{Ethical Considerations}

In this work, we propose the DisAAD framework to estimate the uncertainty of black-box LLM via a lightweight proxy model. The datasets we sourced from are publicly available, ensuring transparency and reproducibility of the distillation dataset construction. We do not expect any direct ethical concern from our work, as the framework solely aims to enhance the reliability of black-box LLM outputs by quantifying uncertainty, rather than modifying or exploiting the target LLMs in harmful ways.

\section*{Acknowledgments}

This work is supported by the National Key Research and Development Program of China (2025YFF0515600) and the National Natural Science Foundation of China (62376193). We thank Zongbo Han (BUPT), Jingdong Chen (TJU), Yitao He (TJU) and Haiyun Yao (TJU) for their helpful discussion about the project and comments on the manuscript. The authors appreciate the valuable feedback from anonymous reviewers.

\bibliography{custom}

\newpage

\appendix
\onecolumn

\section*{\centering Appendix}

\section{Theoretical Analysis}

\subsection{Proof of Lipschitz Continuity}

\begin{theorem}[Lipschitz Continuity]
For the distilled model $G: \mathcal{X} \to \mathcal{Y}$, there exists a constant $L > 0$ such that for any $x_1, x_2 \in \mathcal{X}$: $\|G(x_1) - G(x_2)\| \leq L\|x_1 - x_2\|$.
\end{theorem}

\begin{proof}
Let $W$ denote the proxy model's weight matrix. It employs low-rank adaptation (LoRA) on top of base model, with weight matrix:

\begin{equation}
W = W_0 + BA,
\label{eq1}
\end{equation}

\noindent{where $W_0$ represents the base model weights, and $BA$ represents the low-rank adaptation with $B \in \mathbb{R}^{d \times r}$ and $A \in \mathbb{R}^{r \times k}$. Then, the output difference between the proxy model and black-box LLM under the given inputs $x_1, x_2$ can be expressed as:}

\begin{equation}
\begin{aligned}
\|G(x_1) - G(x_2)\| &= \|(W_0 + BA)(x_1 - x_2)\| \\
&\leq \|W_0(x_1 - x_2)\| + \|BA(x_1 - x_2)\|.
\end{aligned}
\label{eq62}
\end{equation}

For the LoRA component, it satisfies the submultiplicative property of matrix norms:


\begin{equation}
\begin{aligned}
\|G(x_1) - G(x_2)\| 
&\leq \left( \|W_0\| + \|B\| \cdot \|A\| \right) \|x_1 - x_2\|. 
\end{aligned}
\label{eq_merged}
\end{equation}

Let $\alpha$ and $r$ denote the LoRA scaling factor and the adaptation rank, respectively. Given the $\|BA\| \leq \tfrac{\alpha}{r}$, $\|W_0\| + \|B\| \cdot \|A\|$ is a constant, thus we obtain: 

\begin{equation}
\|G(x_1) - G(x_2)\| \leq L\|x_1 - x_2\|,
\end{equation}

\noindent{where $L = \|W_0\| + \|B\| \cdot \|A\|$ serves as our Lipschitz constant.} 
 
\end{proof}

\subsection{Proof of Adversarial Distillation with Limited Samples}

First, we offer the definitions of missing mass and Zipf distribution for LLMs.

\begin{definition}[Missing Mass]
Given the input prompt $x$, the corresponding true output distribution $P_{\text{B}}(\boldsymbol{y}|\boldsymbol{x})$ of a black-box LLM and $k$ samples $\{\boldsymbol{y}_\text{B}^{(i)}\}_{i=1}^k$ drawn from $P_{\text{B}}(\boldsymbol{y}|\boldsymbol{x})$, the missing mass $U_k$ is defined as:

\begin{equation}
U_k = \sum_{\boldsymbol{y} \in \mathcal{Y}} P_{\text{B}}(\boldsymbol{y}|\boldsymbol{x}) \cdot \mathbb{I}\{\boldsymbol{y} \notin \{\boldsymbol{y}_\text{B}^{(1)}, \ldots, \boldsymbol{y}_\text{B}^{(k)}\}\},
\end{equation}

\noindent{where $U_k$ represents the total probability of responses not observed in the $k$ samples, $\boldsymbol{y}$ represents a possible model response, $\mathcal{Y}$ is the set of all possible responses.}

\end{definition}

\begin{definition}[Zipf Distribution and Concentration Function]
The output distribution of one LLM typically follows a Zipf-like power law where the probability of the $i$-th most likely response is proportional to $i^{-\alpha}$ for some $\alpha > 1$:

\begin{equation}
P_{\text{B}}(\boldsymbol{y}_i|\boldsymbol{x}) \propto i^{-\alpha}.
\end{equation}

The concentration function $H(v)$ for a distribution is defined as:

\begin{equation}
H(v) = \sum_{P_{\text{B}}(\boldsymbol{y}|\boldsymbol{x}) \leq v} P_{\text{B}}(\boldsymbol{y}|\boldsymbol{x}).
\end{equation}
\end{definition}

To prove our main theorem, we need the following lemmas.

\begin{lemma}
For distributions following Zipf's law with parameter $\alpha > 1$, the parameter $\beta$ defined as:

\begin{equation}
\beta = \liminf_{v \to 0} \frac{\ln H(v)}{\ln v}, 
\end{equation}

\noindent{where the parameter $\beta$ satisfies $\beta \geq \frac{\alpha-1}{\alpha}$. Furthermore, for any $\varepsilon > 0$, there exists $k_0$ such that for all $k > k_0$:}

\begin{equation}
\mathbb{E}[U_k] \leq k^{-(\beta-\varepsilon)}.
\end{equation}
\end{lemma}

\begin{proof}

For distributions following Zipf's law, the concentration function exhibits specific scaling properties near zero. Using the theory of regular variation, for a Zipf distribution with parameter $\alpha > 1$, it can be shown that $H(v) \sim v^{(\alpha-1)/\alpha}$ as $v \to 0$. This directly leads to $\beta \geq \frac{\alpha-1}{\alpha}$.

The bound on $\mathbb{E}[U_k]$ follows from established results in the concentration function theory for discrete distributions with heavy tails. Specifically, the expected missing mass decays polynomially with the sample size at a rate determined by the parameter $\beta$.
\end{proof}

\begin{lemma}
Using Hoeffding's inequality, with probability at least $1-\delta/2$:
\begin{equation}
|U_k - \mathbb{E}[U_k]| \leq \sqrt{\frac{\ln(2/\delta)}{2k}}.
\end{equation}
\end{lemma}

\begin{proof}
The missing mass $U_k$ can be viewed as a function of $k$ independent samples from $P_{\text{B}}(\boldsymbol{y}|\boldsymbol{x})$. This function satisfies a bounded differences condition: changing any single sample can change the value of $U_k$ by at most $\frac{1}{k}$. Applying Hoeffding's inequality for such functions yields the stated result.
\end{proof}

\begin{lemma}
For the empirical distribution $\hat{P}_{\text{B}}$ constructed from $k$ samples:
\begin{equation}
\mathcal{D}_{\text{KL}}(P_{\text{B}} \parallel \hat{P}_{\text{B}}) \leq -\ln(1 - U_k) + \frac{1}{1-U_k}\sum_{i=1}^k \frac{|f_i - p_i|}{p_i},
\end{equation}

\noindent{where $p_i = P_{\text{B}}(\boldsymbol{y}_\text{B}^{(i)}|\boldsymbol{x})$ and $f_i$ is the empirical frequency of $\boldsymbol{y}_\text{B}^{(i)}$.}
\end{lemma}

\begin{proof}
We decompose the KL divergence between the true distribution $P_{\text{B}}$ and the empirical distribution $\hat{P}_{\text{B}}$:

\begin{align}
\mathcal{D}_{\text{KL}}(P_{\text{B}} \parallel \hat{P}_{\text{B}}) &= \sum_{\boldsymbol{y} \in \mathcal{Y}} P_{\text{B}}(\boldsymbol{y}|\boldsymbol{x}) \ln\frac{P_{\text{B}}(\boldsymbol{y}|\boldsymbol{x})}{\hat{P}_{\text{B}}(\boldsymbol{y}|\boldsymbol{x})} \\
&= \sum_{\boldsymbol{y} \notin \{\boldsymbol{y}_\text{B}^{(i)}\}_{i=1}^k} P_{\text{B}}(\boldsymbol{y}|\boldsymbol{x}) \ln\frac{P_{\text{B}}(\boldsymbol{y}|\boldsymbol{x})}{\hat{P}_{\text{B}}(\boldsymbol{y}|\boldsymbol{x})} + \sum_{i=1}^k P_{\text{B}}(\boldsymbol{y}_\text{B}^{(i)}|\boldsymbol{x}) \ln\frac{P_{\text{B}}(\boldsymbol{y}_\text{B}^{(i)}|\boldsymbol{x})}{\hat{P}_{\text{B}}(\boldsymbol{y}_\text{B}^{(i)}|\boldsymbol{x})}.
\end{align}

For the first term, note that $\hat{P}_{\text{B}}(\boldsymbol{y}|\boldsymbol{x}) = 0$ for all $\boldsymbol{y} \notin \{\boldsymbol{y}_\text{B}^{(i)}\}_{i=1}^k$,  a standard approach involves using a smoothed version of the empirical distribution to handle the undefined logarithm, which leads to a bound related to the missing mass $U_k$. This yields:

\begin{equation}
\sum_{\boldsymbol{y} \notin \{\boldsymbol{y}_\text{B}^{(i)}\}_{i=1}^k} P_{\text{B}}(\boldsymbol{y}|\boldsymbol{x}) \ln\frac{P_{\text{B}}(\boldsymbol{y}|\boldsymbol{x})}{\hat{P}_{\text{B}}(\boldsymbol{y}|\boldsymbol{x})} \leq -\ln(1 - U_k).
\end{equation}

For the second term, by applying convexity arguments and properties of logarithms, one can derive the following bound for the observed samples:

\begin{equation}
\sum_{i=1}^k P_{\text{B}}(\boldsymbol{y}_\text{B}^{(i)}|\boldsymbol{x}) \ln\frac{P_{\text{B}}(\boldsymbol{y}_\text{B}^{(i)}|\boldsymbol{x})}{\hat{P}_{\text{B}}(\boldsymbol{y}_\text{B}^{(i)}|\boldsymbol{x})} \leq \frac{1}{1-U_k}\sum_{i=1}^k \frac{|f_i - p_i|}{p_i}.
\end{equation}

Combining these bounds gives us the desired result.
\end{proof}

\begin{lemma}
With probability at least $1-\delta/2$:
\begin{equation}
\sum_{i=1}^k \frac{|f_i - p_i|}{p_i} \leq C_2\sqrt{\frac{\ln(2/\delta)}{k}},
\end{equation}
where $C_2$ is a constant.
\end{lemma}

\begin{proof}
This follows from standard concentration inequalities for multinomial distributions. The empirical frequencies $f_i$ are unbiased estimators of the true probabilities $p_i$, and their deviations can be bounded using results from statistical learning theory.
\end{proof}

\begin{lemma}
Using the bound on $U_k$ and the inequality $-\ln(1-z) \leq \frac{z}{1-z}$ for $z \in [0,1)$:
\begin{equation}
-\ln(1 - U_k) \leq \frac{U_k}{1-U_k} \leq \frac{C_1}{k^{\gamma}},
\end{equation}
where $\gamma = \beta - \varepsilon$ and $C_1$ is a constant.
\end{lemma}

\begin{proof}
From Lemma 1 and Lemma 2, we have with probability at least $1-\delta/2$:
\begin{equation}
U_k \leq \mathbb{E}[U_k] + \sqrt{\frac{\ln(2/\delta)}{2k}} \leq k^{-(\beta-\varepsilon)} + \sqrt{\frac{\ln(2/\delta)}{2k}}.
\end{equation}

For a typical Zipf distribution in LLM, $\beta > 1/2$, making the polynomial term $k^{-(\beta-\varepsilon)}$ the dominant one for sufficiently large $k$. Thus, we can bound $U_k \leq C'k^{-(\beta-\varepsilon)}$ for some constant $C'$. Applying the inequality $-\ln(1-z) \leq \frac{z}{1-z}$, we get:

\begin{equation}
-\ln(1 - U_k) \leq \frac{U_k}{1-U_k} \leq \frac{C'k^{-(\beta-\varepsilon)}}{1-C'k^{-(\beta-\varepsilon)}} \leq \frac{C_1}{k^{\gamma}},
\end{equation}

\noindent{where $\gamma = \beta - \varepsilon$ and $C_1$ is a constant that depends on $C'$.}
\end{proof}

Now we are ready to present the proof of our main theorem.

\begin{theorem}[Sample-Finited Adversarial Distillation]
Let $P_{\text{B}}(\boldsymbol{y}|\boldsymbol{x})$ be the true output distribution of a black-box LLM $\mathcal{M}_\text{B}$ for input prompt $\boldsymbol{x}$. For $k$ samples $\{\boldsymbol{y}_\text{B}^{(i)}\}_{i=1}^k$ drawn from $P_{\text{B}}(\boldsymbol{y}|\boldsymbol{x})$, the empirical distribution $\hat{P}_{\text{B}}(\boldsymbol{y}|\boldsymbol{x})$ constructed from these samples satisfies, with probability at least $1-\delta$:

\begin{equation}
\mathcal{D}_{\text{KL}}(P_{\text{B}}(\boldsymbol{y}|\boldsymbol{x}) \parallel \hat{P}_{\text{B}}(\boldsymbol{y}|\boldsymbol{x})) \leq \frac{C_1}{k^{\gamma}} + C_2\sqrt{\frac{\ln(1/\delta)}{k}},
\end{equation}

\noindent{where $\gamma > 0$ depends on the power-law characteristics of $P_{\text{B}}$, and $C_1, C_2$ are constants.}
\end{theorem}

\begin{proof}
By combining Lemma 3 with Lemmas 4 and 5, and using the union bound, we get with probability at least $1-\delta$:

\begin{align}
\mathcal{D}_{\text{KL}}(P_{\text{B}} \parallel \hat{P}_{\text{B}}) &\leq -\ln(1 - U_k) + \frac{1}{1-U_k}\sum_{i=1}^k \frac{|f_i - p_i|}{p_i} \\
&\leq \frac{C_1}{k^{\gamma}} + C_2\sqrt{\frac{\ln(2/\delta)}{k}},
\end{align}

\noindent{where $\gamma = \beta - \varepsilon \geq \frac{\alpha-1}{\alpha} - \varepsilon$, and $C_1, C_2$ are constants. Since $U_k \to 0$ as $k \to \infty$, the term $1/(1-U_k)$ approaches 1 and can be absorbed into a new constant. After adjusting constants to reflect the total probability bound of $1-\delta$, we arrive at the final result shown in the theorem.}

\end{proof}

\twocolumn

\section{Prompt Templates}

We provide the specific prompt templates used in our experiments for both response generation and evaluation.

\subsection{Prompts for Response Generation and Reliability Estimation}

Following the previous works, we utilize distinct prompt structures tailored to the specific formatting requirements of each for the LLaMa2 series, LLaMa3 series, Qwen3 series and GPT-4, respectively. The \{question\} placeholder is dynamically replaced with the input question from the question-answer datasets.

\definecolor{myblack}{RGB}{0,0,0}
\definecolor{myred}{RGB}{200,0,0}
\definecolor{mygray}{gray}{0.97}
\lstdefinestyle{promptstyle}{
  basicstyle=\ttfamily\normalsize,
  backgroundcolor=\color{mygray},
  breaklines=true,
  columns=fullflexible,
  escapeinside={(*@}{@*)},  
  showstringspaces=false,
  frame=none
}

\newtcolorbox{blacktitlebox}[1]{
  colback=white,
  colframe=black,
  coltitle=white,
  title=#1,
  fonttitle=\bfseries,
  colbacktitle=black,
  boxrule=0.5pt,
  toptitle=1mm,
  bottomtitle=1mm,
}

\begin{blacktitlebox}{Prompt for Response Reliability Estimation}

\textbf{LLaMa2 Series Prompt}
\begin{lstlisting}[style=promptstyle, numbers=none]
Answer the question concisely. Q: (*@\textcolor{red}{\{question\}}@*) A:
\end{lstlisting}

\textbf{LLaMa3 Series Prompt}
\begin{lstlisting}[style=promptstyle, numbers=none]
<|eot_id|>
<|start_header_id|>user<|end_header_id|>
Answer the question concisely. Q: (*@\textcolor{red}{\{question\}}@*) A:<|eot_id|>
\end{lstlisting}

\textbf{GPT-4 Prompt}
\begin{lstlisting}[style=promptstyle, numbers=none]
{"role": "user", "content": "Answer the question concisely. Q: (*@\textcolor{red}{\{question\}}@*) A:"}
\end{lstlisting}

\textbf{Qwen3 Series Prompt}
\begin{lstlisting}[style=promptstyle, numbers=none]
<|im_start|>user
Answer the question concisely.  
Q: (*@\textcolor{red}{\{question\}}@*) A:<|im_end|>
<|im_start|>assistant
\end{lstlisting}

\end{blacktitlebox}

\subsection{System Prompt for LLM-as-a-Judge}

To automatically evaluate the correctness of generated responses for datasets like BioASQ and TriviaQA, we employ an LLM-as-a-Judge approach. The following system prompts are used to instruct the judge model (e.g., GPT-4) to act as an impartial evaluator. The prompt provides clear instructions, examples of both correct (1) and incorrect (0) evaluations, and the final template used for batch processing.

\begin{blacktitlebox}{System Prompt for LLM-as-Judge}

\texttt{\textbf{System:} Your task is to determine if the provided answer is true or false based solely on the ground truth answers given to you in the format [’answer 1’, ’answer 2’, ...]. DO NOT rely on your memory; only use the information provided after this instruction. Respond with 1 if the predicted answer is correct, which means semantically consistent with any of the ground truth answers, otherwise respond with 0. Respond with just 0 or 1, and DO NOT include anything else in your response. This is the only instruction you need to follow.}

\texttt{\textbf{User:} Input: Who is elected as the vice president of india in 2017?}

\texttt{\textbf{Ground Truth:} [‘Venkaiah Naidu’, ‘Muppavarapu Venkaiah Naidu’]}

\texttt{\textbf{Provided Answer:} M. Venkaiah Naidu}

\texttt{\textbf{Assistant:}1}

\texttt{\textbf{User:} Input: who sings you are a magnet and i am steel?}

\texttt{\textbf{Ground Truth:} [‘Walter Egan’]}

\texttt{\textbf{Provided Answer:} The song ‘You Are a Magnet and I Am Steel’ is performed by the band The 1975.}

\texttt{\textbf{Assistant:}0}
\vspace{1em}

\texttt{\textbf{User:}Input: \textcolor{red}{\{Question\}}}

\texttt{\textbf{Ground Truth:} \textcolor{red}{\{Your Ground Truth List\}}}

\texttt{\textbf{Provided Answer:} \textcolor{red}{\{The Answer to be Judged\}}}

\texttt{\textbf{Assistant:}}
\end{blacktitlebox}

\section{Method Description}

\subsection{Distillation Dataset Generation}

This section provides a detailed description of our proposed methodology. Algorithm \ref{alg1} outlines the process for constructing our high-quality distillation dataset. This process begins by creating a diverse set of prompts, collated from both large-scale conversational WildChat (open-domain) and task-specific evaluation dataset (in-domain). For each prompt, we leverage the black-box LLM to generate a set of candidate responses using a dual-temperature sampling strategy to capture both precision and diversity. These responses then undergo a rigorous filtering process to ensure that each prompt in the final distillation dataset is paired with a set of representative responses.

\begin{algorithm}[!htbp]
\caption{Distillation Dataset Generation}
\label{alg1}
\textbf{Input}: A collection of $N$ prompts from mixed multi-source datasets\\
\textbf{Model}: Black-box LLM $\mathcal{M}_\text{B}$ with specified generation parameters (e.g., temperature, top-$M$ sampling) \\
\textbf{Output}: Distillation dataset derived from target LLM
\begin{algorithmic}[1]
\STATE Sample $N/2$ prompts from WildChat and $N/2$ from the evaluation dataset
\STATE Merge and shuffle to form the mixed prompt set
\FOR{each prompt}
    \STATE Generate one response with low temperature ($T \approx 0$) using the black-box LLM
    \STATE Generate multiple responses with high temperature ($T > 0.5$) using the black-box LLM 
\ENDFOR
\FOR{each prompt and its response set}
    \STATE Discard high-temp responses that are too short, repetitive, or high in perplexity
    \STATE Keep $1$ low-temp + up to $M-1$ high-quality high-temp responses
    \IF{valid response count $< M$}
        \STATE Discard the sample
    \ENDIF
\ENDFOR
\STATE \textbf{return} 
\end{algorithmic}

\end{algorithm}

\subsection{Proxy Model Optimization}

Algorithm \ref{alg2} describes the adversarial distillation process for optimizing the proxy model $\mathcal{M}_\text{p}$ using the proposed DisAAD framework. It establishes an adversarial training dynamic between the LoRA-based proxy model (the generator) and the discriminator $\mathcal{M}_\text{D}$. The training objective is to align the proxy model's responses with the high-probability region of the output distribution of the target black-box LLM. 

To maintain a stable and effective adversarial process, the training proceeds in a carefully balanced alternating fashion: for every single update to the proxy model, the discriminator is first updated twice, which ensures that the discriminator can provide a robust and informative learning signal by refining its ability to distinguish between responses from the black-box model's distribution (real) and those from the proxy model's distribution (fake). Following this, the proxy model's trainable parameters are updated based on a composite loss function, which combines a standard token-level task loss with a sequence-level adversarial loss provided by the discriminator. This iterative process drives the proxy model to generate responses that are increasingly indistinguishable from those of the target black-box LLM.

\begin{algorithm}[!htbp]
\caption{The Optimized Proxy Model using DisAAD}
\label{alg2}
\textbf{Input}: Supervised distillation dataset $\{ (\boldsymbol{x}^{(i)}, \{\boldsymbol{y}_\text{B}^{(i,j)}\}_{j=1}^M) \}_{i=1}^{N}$ \\
\textbf{Model}: Proxy model $\mathcal{M}_\text{p}$ with LoRA parameters $\theta$; Discriminator $\mathcal{M}_\text{D}$ with parameters $\phi$ \\
\textbf{Output}: The optimized proxy model $\mathcal{M}_\text{p}$ with optimized parameters $\theta$ 
\begin{algorithmic}[1]
\STATE Divide $\mathcal{D}_\text{distill}$ into $\mathcal{D}_\text{train}$ and $\mathcal{D}_\text{val}$
\FOR{each training iteration}
    \STATE Sample a batch of prompts $\{\boldsymbol{x}^{(i)}\}_{i \in \mathcal{B}}$ from $\mathcal{D}_\text{train}$, and retrieve all corresponding target responses $\{\{\boldsymbol{y}_\text{B}^{(i,j)}\}_{j=1}^M\}_{i \in \mathcal{B}}$
    \STATE Generate $M$ responses for each prompt: $\{\{\boldsymbol{y}_\text{p}^{(i,j)}\}_{j=1}^M\}_{i \in \mathcal{B}} \gets \mathcal{M}_\text{p}(\{\boldsymbol{x}^{(i)}\}_{i \in \mathcal{B}}; \theta)$
    
    \STATE \textit{// Update Discriminator $\mathcal{M}_\text{D}$}
    \STATE Define real pairs $\mathcal{P}_\text{real} \gets \{ (\boldsymbol{x}^{(i)}, \boldsymbol{y}_\text{B}^{(i,j)}) \mid i \in \mathcal{B}, j \in [1, M] \}$
    \STATE Define fake pairs $\mathcal{P}_\text{fake} \gets \{ (\boldsymbol{x}^{(i)}, \boldsymbol{y}_\text{p}^{(i,j)}) \mid i \in \mathcal{B}, j \in [1, M] \}$
    \STATE Compute discriminator loss $\mathcal{L}_\text{D}(\phi)$ over all pairs in $\mathcal{P}_\text{real}$ and $\mathcal{P}_\text{fake}$
    \STATE Update discriminator parameters: $\phi \gets \phi - \eta_D \nabla_\phi \mathcal{L}_\text{D}(\phi)$
    
    \STATE \textit{// Update Proxy Model $\mathcal{M}_\text{p}$}
    \STATE Compute task loss $\mathcal{L}_\text{task}(\theta)$ over the set of real pairs $\mathcal{P}_\text{real}$
    \STATE Compute regularization (adversarial) loss $\mathcal{L}_\text{reg}(\theta)$ over the set of fake pairs $\mathcal{P}_\text{fake}$
    \STATE Aggregate total loss $\mathcal{L}(\theta) \gets \mathcal{L}_\text{task}(\theta) + \lambda \mathcal{L}_\text{reg}(\theta)$
    \STATE Update proxy model's LoRA parameters: $\theta \gets \theta - \eta_p \nabla_\theta \mathcal{L}(\theta)$
    
    \STATE Periodically validate $\mathcal{M}_\text{p}$ on $\mathcal{D}_\text{val}$ and save the best-performing checkpoint
\ENDFOR
\STATE \textbf{return} 
\end{algorithmic}
\end{algorithm}

\subsection{Uncertainty Description}

\begin{figure*}[h]
\centering
\includegraphics[width=0.95\textwidth]{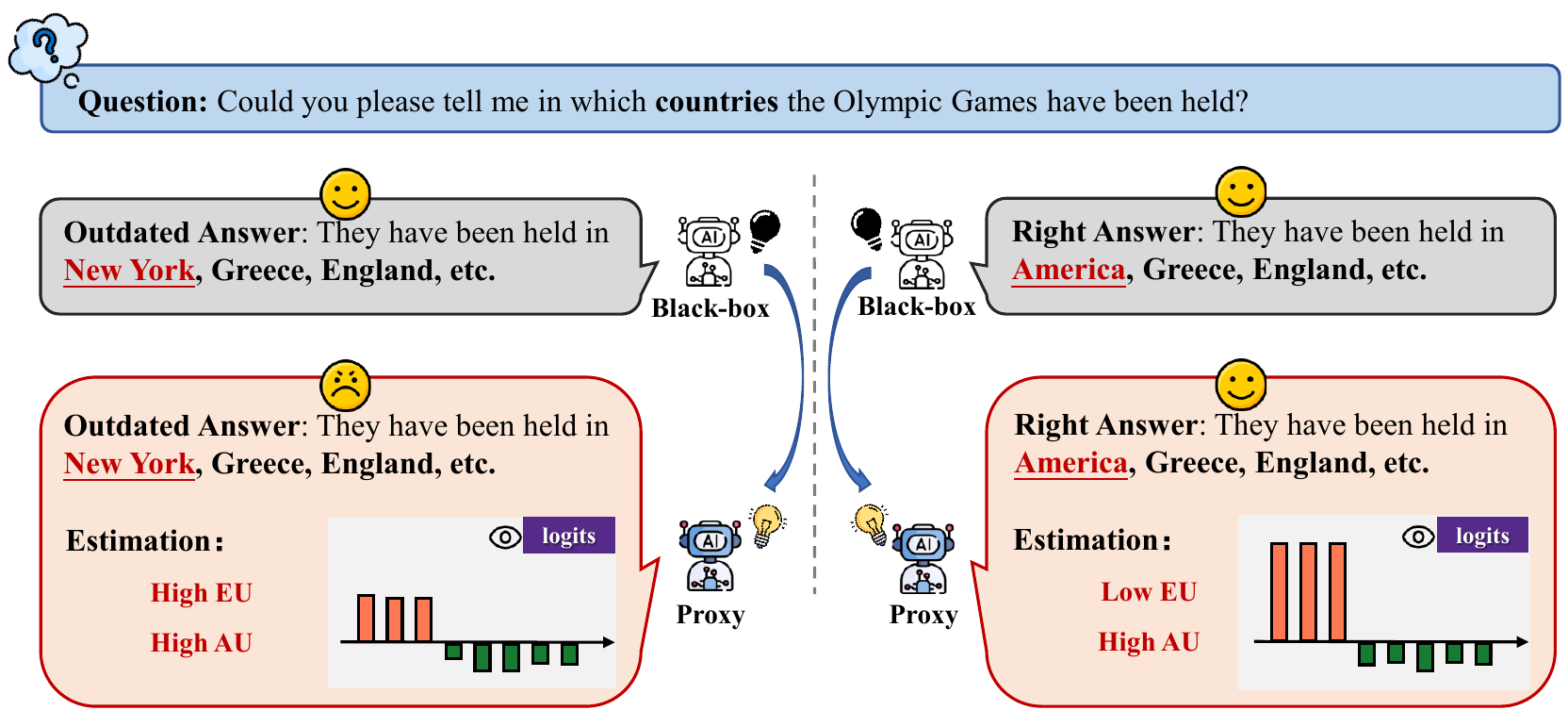} 
\caption{Overview of proxy-guided uncertainty quantification. For a given response from the target LLM, we first utilize the proxy model to reproduce the responses, then extract the corresponding logits and estimate the decoupled uncertainty. For an incorrect answer (e.g., ``New York''), the proxy model identifies this scenario as High EU and High AU, since the LLM mistakenly identifies ``New York'' as one of the countries that have held the Olympic Games. Conversely, for a correct answer (e.g., ``America''), the proxy model identifies this scenario as Low EU and High AU. This combination suggests that the target model is confident in its knowledge (Low EU) while also recognizing that the question has multiple correct answers (High AU), thus indicating the response is reliable.}
\label{flow3}
\end{figure*}

Here, we also give more explanation about the decoupled uncertainty based on evidential deep learning. As shown in Fig. \ref{flow2} and Fig. \ref{flow3}, the proposed work can effectively quantify the LLM uncertainty even when there are multiple correct responses. Four different scenarios considered for black-box LLM uncertainty are listed as follows: 

\begin{itemize}
    \item High AU, High EU: the black-box LLM lacks domain knowledge and produces uncertain predictions; 
    \item Low AU, High EU: it is an overconfidence scenario where the target LLM generates confident prediction despite knowledge gaps; 
    \item Low AU, Low EU: it is an optimal reliability scenario with both strong knowledge and high prediction confidence; 
    \item High AU, Low EU: the black-box LLM knows more than one reasonable answer.
\end{itemize}

\section{More Experimental Details}

\subsection{Dataset Description}

\subsubsection{TruthfulQA}

TruthfulQA is a benchmark designed to evaluate the truthfulness of language model outputs. It contains questions that are adversarially selected to elicit false or misleading answers from models. The dataset spans multiple categories such as health, law, and finance, emphasizing factual consistency over plausibility. Each question has a reference answer annotated for truthfulness, making it suitable for both uncertainty and calibration evaluation.

\subsubsection{BioASQ}
BioASQ is a biomedical question answering benchmark comprising expert-annotated questions based on PubMed articles. We use the factoid subset of the dataset, which consists of questions with short, factual answers. BioASQ is particularly challenging due to domain-specific terminology and the requirement for precise biomedical knowledge. Each question is paired with a gold-standard answer list, enabling both exact match and ranking-based evaluation.

\subsubsection{TriviaQA}
TriviaQA is an open-domain question answering dataset containing over $650K$ question-answer pairs sourced from trivia websites and verified using evidence documents from Wikipedia. The questions are naturally complex and often require multi-hop reasoning. We use the unfiltered open-domain version, which pairs each question with multiple evidence documents, making it suitable for evaluating answer faithfulness and uncertainty under broader context exposure.

\subsection{Evaluation Metrics}

The evaluation is formulated as a binary classification task, and conducted on question-answering benchmarks from different domains, including life sciences (BioASQ \citep{tsatsaronis2015overview}), truthfulness (TruthfulQA \citep{lin2021truthfulqa}) and knowledge seeking (TriviaQA \citep{joshi2017triviaqa}). The responses generated by LLM with ``$\text{BLEURT}$$>$0.5'' or ``LLM-Judge=1'' are considered the correct answers \citep{xiong2024efficient}. During the testing phase, for each dataset, we randomly selected $800$ samples to evaluate the performance of different methods. The performance is quantified by the Area Under the Receiver Operating Characteristic curve (AUROC) and the Area Under the Precision-Recall curve (AUPR) assess its discriminative power, while the Expected Calibration Error (ECE) evaluates its calibration accuracy.

\subsection{Target and Proxy Models} 

We employ Llama2-70B-Instruct, Qwen3-32B and GPT-4 (i.e., GPT-4-0613) as the target LLMs, and utilize Llama3-1B-Instruct, Llama3-3B-Instruct, Llama3-8B-Instruct as proxy models. Although our method is adaptable to various open-source proxy models, we standardize on the LLaMA3 series to ensure a controlled and consistent basis for our comparative analysis. We run all the experiments on 2-4 NVIDIA GeForce RTX 4090 GPUs with parallel processing. It is worth noting that we did not carefully select the hyperparameters, we believe that by making careful adjustments, better results can be achieved.

\subsection{Baseline Methods}

We demonstrate the response reliability estimation performance of our proposed framework in both black-box and white-box settings. For the multi-sample methods applicable in the black-box setting, including LN-Entropy (LNE), Semantic Entropy (SE), Discrete Semantic Entropy (DSE), and Lexical Similarity (LeS), we follow the default setting in the original paper by generating $10$ candidate responses with a temperature of $0.5$ to derive the uncertainty scores. For LeS, Rouge-L is utilized as the similarity metric, whereas for SE and DSE, Deberta-Large-MNLI model is employed to form semantic clusters. Furthermore, for the single-sample methods applicable in the white-box setting, such as probability, entropy and LogTokU, we produce the single most probable response via greedy decoding strategy, then retain the token-level logits and compute the uncertainty scores.

\subsection{Distillation Dataset Collection}

To construct a distillation dataset that is both diverse and domain-specific, we strategically combine prompts from in-domain and out-of-domain datasets. Half are drawn from a general dataset rich in real-world interactions (e.g., WildChat), while the other half originate from the relevant evaluation dataset. This hybrid approach balances domain relevance with generalizability, enabling the proxy model to align closely with the target LLM's output distribution. For each input prompt, we generate $15$ candidate responses: one response with low temperature ($T \approx 0.01$) and $14$ diverse responses with high temperature ($T>0.5$). These candidates then undergo a systematic filtering process. First, we discard high-temperature responses that fail baseline quality checks, specifically insufficient length (e.g., fewer than $15$ words) or excessive sequence repetition. Following this, we compute the cosine similarity of each valid high-temperature response to its corresponding low-temperature response. We then retain the low-temperature response and the top nine high-temperature responses based on this similarity ranking. An input prompt is discarded entirely if it fails to yield this complete set of $10$ valid responses after this procedure. This procedure is designed to capture the characteristic output distribution of the black-box model, rather than to filter exclusively for factual correctness. The final dataset consists of these newly generated responses paired with their corresponding original prompts, forming a collection of prompt-response pairs. Our experiments demonstrate that an effective proxy model can be obtained based on the proposed DisAAD with merely $1K$ samples, making our method not only powerful but also computationally efficient.

\section{More Experimental Results}

\begin{table*}[htbp]
\centering
\resizebox{2\columnwidth}{!}{
\begin{tabular}{l *{7}{cc}}
\toprule
\multirow{2.5}{*}{Method} 
& \multicolumn{2}{c}{$K=1$} 
& \multicolumn{2}{c}{$K=5$} 
& \multicolumn{2}{c}{$K=10$} 
& \multicolumn{2}{c}{$K=15$} 
& \multicolumn{2}{c}{$K=20$} 
& \multicolumn{2}{c}{$K=25$} 
& \multicolumn{2}{c}{$K=ALL$} \\
\cmidrule(lr){2-3} \cmidrule(lr){4-5} \cmidrule(lr){6-7} \cmidrule(lr){8-9} \cmidrule(lr){10-11} \cmidrule(lr){12-13} \cmidrule(lr){14-15}
 & AUROC$\uparrow$ & AUPR$\uparrow$ & AUROC$\uparrow$ & AUPR$\uparrow$ & AUROC$\uparrow$ & AUPR$\uparrow$ 
& AUROC$\uparrow$ & AUPR$\uparrow$ & AUROC$\uparrow$ & AUPR$\uparrow$ & AUROC$\uparrow$ & AUPR$\uparrow$
& AUROC$\uparrow$ & AUPR$\uparrow$ \\
\midrule
Probability     & 0.5503 & 0.6734 & 0.5739 & 0.6785 & 0.5611 & 0.6751 & 0.5548 & 0.6747 & 0.5533 & 0.6746 & 0.5533 & 0.6727 & 0.5707 & 0.6720 \\
Entropy  & 0.5849 & 0.6946 & 0.6045 & 0.7010 & 0.5907 & 0.6945 & 0.5798 & 0.6908 & 0.5757 & 0.6914 & 0.5738 & 0.6899 & 0.5967 & 0.6924 \\
Logtoku  & 0.5898 & 0.6943 & 0.6320 & 0.7185 & 0.6370 & 0.7178 & 0.6351 & 0.7189 & 0.6326 & 0.7177 & 0.6324 & 0.7199 & 0.6256 & 0.7253 \\
DisAAD   & \textbf{0.7046} & \textbf{0.7848} & \textbf{0.7105} & \textbf{0.7868} & \textbf{0.7038} & \textbf{0.7808} & \textbf{0.6946} & \textbf{0.7756} & \textbf{0.6905} & \textbf{0.7764} & \textbf{0.6866} & \textbf{0.7765} & \textbf{0.7046} & \textbf{0.7848} \\
\bottomrule
\end{tabular}}
\caption{The comparison reliability estimation performance of different numbers of tokens with high uncertainty on BioASQ dataset. The proxy model and the target LLM are LLaMA3-3B and LLaMA2-70B, respectively. ``K=ALL'' denotes all the generated tokens (128 tokens) are used for estimation.}
\label{ab2}
\end{table*}

\subsection{Efficiency of the Small-Size Distillation Dataset and Discriminator}

To demonstrate the training efficiency of our proposed DisAAD framework, we analyze the performance of small-size distillation dataset and the generation-discrimination architecture. As shown in Fig. \ref{ab1}(a), both AUROC and AUPR metrics rapidly improve with initial samples and stabilize after approximately $50$ distillation prompts (equivalent to $0.5K$ distillation data). Notably, performance slightly decreases before final convergence, suggesting that the proxy model first captures essential uncertainty patterns from limited data, then experiences minor overfitting to noise present in larger datasets. The calibration results in Fig. \ref{ab1}(b) further support this finding, as the ECE values remain consistent across different distillation dataset sizes, indicating that additional data beyond the optimal point does not meaningfully improve calibration quality.
 
Furthermore, the effectiveness of the generation-discrimination architecture is demonstrated in Fig. \ref{ab2}, where we employ LLaMA3-3B as the generator and GPT-4 as the discriminator. To quantitatively measure the performance of the distilled proxy model, we define the prediction gap as the difference in discriminator confidence when classifying outputs from the proxy model versus the target LLM. As shown in Fig. \ref{ab2}(a), this gap progressively narrows throughout the training process, eventually stabilizing at approximately $0.0050$. This convergence indicates the discriminator can no longer effectively distinguish between the outputs of both models, confirming that the proxy model has successfully learned to mimic the target LLM's response patterns. The semantic similarity analysis in Fig. \ref{ab2}(b) further supports this finding, illustrating the distribution of similarity scores for a number of $100$ validation samples at different training steps. The clear shift toward similarity scores exceeding $0.9$ between steps $30$ and $70$ demonstrates that the proxy model achieves high consistency with the target model's output distribution. Overall, empirical results demonstrate that the proposed DisAAD framework requires only a small distillation dataset, with the proxy model effectively approximating the target LLM's performance during the early stages of adversarial training, thus substantially reducing the query cost associated with distillation.

\subsection{Discussion of the Choice of Top-$K$ Tokens with High Uncertainty}

Table \ref{ab2} presents a comparative analysis of reliability estimation results using the Top-$K$ tokens with the highest uncertainty. It is observed that for all tested choices of $K$, our DisAAD method maintains a substantial performance margin over all baseline methods. In addition, the reliability estimation achieved using a small subset of tokens (e.g., $K=5$ or $K=10$) is not only comparable to but even slightly superior to that achieved using the entire sequence of generated tokens ($K=\text{ALL}$). This observation aligns with recent research suggesting that during an LLM's inference, only a small portion of tokens are important for estimating reliability. Therefore, focusing on these key tokens with higher uncertainty can provide more focused and clearer signals, while considering the entire token sequence as a whole might potentially affect the accuracy of the estimation. In this paper, we uniformly set $K$ as $20\%$ of the total token count for the generated response for all datasets.

\subsection{Effectiveness of the Logits-based Uncertainty Estimation}

\begin{table*}[htbp]
\centering
\resizebox{2\columnwidth}{!}{
\begin{tabular}{ccc *4{cc}}
\toprule
\multirow{2.5}{*}{Target} & \multirow{2.5}{*}{Method} & \multirow{2.5}{*}{$O(1)$} & \multicolumn{2}{c}{TruthfulQA} & \multicolumn{2}{c}{BioASQ} & \multicolumn{2}{c}{TriviaQA} & \multicolumn{2}{c}{Average} \\
\cmidrule(lr){4-5}\cmidrule(lr){6-7}\cmidrule(lr){8-9}\cmidrule(lr){10-11}
& & & AUROC$\uparrow$ & AUPR$\uparrow$ & AUROC$\uparrow$ & AUPR$\uparrow$ & AUROC$\uparrow$ & AUPR$\uparrow$ & AUROC$\uparrow$ & AUPR$\uparrow$\\
\midrule
\multirow{3}{*}{LLaMA2-70B} 
& Probability & \ding{51} & 0.7261 & 0.6825 & 0.6482 & 0.7426 & 0.5427 & 0.8175 & 0.6390 & 0.7475 \\
& Entropy     & \ding{51} & 0.7527 & 0.7166 & 0.6763 & 0.7622 & 0.5808 & 0.8465 & 0.6699 & 0.7751 \\
& DisAAD      & \ding{51} & \textbf{0.8015} & \textbf{0.7807} & \textbf{0.7046} & \textbf{0.7874} & \textbf{0.6503} & \textbf{0.8852} & \textbf{0.7188} & \textbf{0.8178} \\
\midrule
\multirow{3}{*}{Qwen3-32B} 
& Probability & \ding{51} & 0.6608 & 0.7369 & 0.6599 & 0.8396 & 0.6852 & 0.7905 & 0.6686 & 0.7890 \\
& Entropy     & \ding{51} & 0.6724 & 0.7710 & \textbf{0.6867} & \textbf{0.8522} & \textbf{0.7127} & \textbf{0.8215} & 0.6906 & 0.8149 \\
& DisAAD      & \ding{51} & \textbf{0.7478} & \textbf{0.8310} & 0.6571 & 0.8480 & 0.6935 & 0.8125 & \textbf{0.6995} & \textbf{0.8305} \\
\midrule
\multirow{3}{*}{GPT-4} 
& Probability & \ding{51} & 0.6803 & 0.8297 & 0.6620 & 0.8701 & 0.6627 & 0.9522 & 0.6683 & 0.8840 \\
& Entropy     & \ding{51} & 0.7370 & 0.8729 & 0.6862 & 0.8758 & 0.6838 & 0.9540 & 0.7023 & 0.9009 \\
& DisAAD      & \ding{51} & \textbf{0.8078} & \textbf{0.9079} & \textbf{0.6993} & \textbf{0.8742} & \textbf{0.6893} & \textbf{0.9580} & \textbf{0.7321} & \textbf{0.9134} \\
\bottomrule
\end{tabular}}
\caption{Comparison of reliability estimation performance based on probability, entropy, and logits-based DisAAD.}
\label{ap1}
\end{table*}

As shown in Table \ref{ap1}, the proposed DisAAD outperforms both probability-based method and entropy-based method in estimating the response reliability in most cases. This superiority is evident from two key observations. First, DisAAD consistently achieves the highest average AUROC and AUPR scores across different types of target LLMs (LLaMA2-70B, Qwen3-32B, and GPT-4), indicating its robust and generalizable performance. For instance, with LLaMA2-70B as the target LLM, DisAAD's average AUROC of $0.7188$ significantly surpasses the $0.6699$ from entropy and $0.6390$ from probability. Theoretically, this advantage stems from the nature of logits as raw and unnormalized scores from the model's final layer. Unlike probabilities, which are normalized by the softmax function and thus only reflect the relative differences between scores, logits retain crucial information about the absolute magnitude of the model's conviction. By operating directly on these richer, uncompressed logit representations, DisAAD can access a more fine-grained and reliable signal of the model's true internal state of confidence, ultimately enabling a more accurate distinction between reliable and unreliable answers.

\subsection{Potential Risks}

AI safety is closely related to the reduction of LLM hallucination. These non-factual but seemingly reasonable outputs pose risks for safety-critical applications. Our DisAAD framework estimates the uncertainty of black-box LLMs through a streamlined proxy model, which helps detect illusion phenomena and thereby enhances the security of deployment. However, our proxy model relies on the quality of the dataset, which can lead to deviations in uncertainty estimation and result in the omission of illusion phenomena in critical scenarios. Additionally, the bias of discriminator may also cause the proxy model to overestimate or underestimate the response uncertainty of the black-box LLM. Therefore, users who adopt the proposed method need to be cautious when conducting adversarial distillation of the proxy model.

\end{document}